\makeatletter
\@ifundefined{iffull}{%
	\newif\iffull
	\fulltrue
}{}
\@ifundefined{ifshort}{%
	\newif\ifshort
	\shortfalse
}{}
\makeatother

\iffull
\documentclass{article}
\usepackage[utf8]{inputenc}
\usepackage{amsfonts}
\usepackage{amsthm}
\usepackage{amsmath}
\usepackage{xcolor}
\usepackage{amssymb}
\usepackage{hyperref,fullpage}

\usepackage[style=alphabetic,backend=bibtex,maxbibnames=20,maxcitenames=6,giveninits=true,doi=false,url=false]{biblatex}
\newcommand*{\citet}[1]{\AtNextCite{\AtEachCitekey{\defcounter{maxnames}{2}}} \textcite{#1}}
\newcommand*{\citetall}[1]{\AtNextCite{\AtEachCitekey{\defcounter{maxnames}{999}}} \textcite{#1}}
\newcommand*{\citep}[1]{\cite{#1}}
\newcommand{\citeyearpar}[1]{\cite{#1}}
\bibliography{vf-allrefs-local,adaptive}

\theoremstyle{plain}
\newtheorem{theorem}{Theorem}
\newtheorem*{theorem*}{Theorem}
\newtheorem{lemma}{Lemma}
\newtheorem{definition}{Definition}

\newtheorem{corollary}{Corollary}

\else

\documentclass[sigconf,screen]{acmart}
\setcopyright{rightsretained}
\acmPrice{}
\acmDOI{10.1145/3357713.3384315}
\acmYear{2020}
\copyrightyear{2020}
\acmSubmissionID{stoc20main-p646-p}
\acmISBN{978-1-4503-6979-4/20/06}
\acmConference[STOC '20]{Proceedings of the 52nd Annual ACM SIGACT Symposium on Theory of Computing}{June 22--26, 2020}{Chicago, IL, USA}
\acmBooktitle{Proceedings of the 52nd Annual ACM SIGACT Symposium on Theory of Computing (STOC '20), June 22--26, 2020, Chicago, IL, USA}

\newcommand{\citetall}[1]{\citet*{#1}}

\fi

\usepackage{dsfont}

\newtheorem{claim}{Claim}

\DeclareMathOperator*{\E}{\mathbb{E}}

\renewcommand{\eqref}[1]{Eq.~(\ref{#1})}
\renewcommand{\tilde}{\widetilde}
\renewcommand{\hat}{\widehat}
\newcommand{\err}{\mathrm{err}}
\newcommand{\eps}{\epsilon}
\newcommand{\poly}{\mathrm{poly}}
\newcommand{\dtv}{d_{\mathrm{TV}}}
\newcommand{\STAT}{\mbox{STAT}}
\newcommand{\LR}{\mathrm{LR}}
\newcommand{\COMM}{\mathrm{COMM}}

\newcommand{\remove}[1]{}

\newcommand{\sgn}{\mathrm{sign}}
\newcommand{\A}{\mathcal{A}}
\newcommand{\K}{\mathcal{K}}

\newcommand{\R}{\mathbb{R}}
\newcommand{\supp}{\mathrm{supp}}
\newcommand{\unif}{\mathrm{Unif}}
\newcommand{\tg}{\tilde{\gamma}}

\usepackage{vfmacros}

\newif\ifnotes
\notesfalse

\ifnotes
\newcommand{\yuval}[1]{\textcolor{red}{\textbf{Yuval: {#1}}}}
\newcommand{\vitaly}[1]{\textcolor{blue}{\textbf{Vitaly: {#1}}}}
\else
\newcommand{\yuval}[1]{}
\newcommand{\vitaly}[1]{}
\fi

\iffull
\author{Yuval Dagan\footnotemark[1]\thanks{* Part of the work was done while the author was at Google Research.} \\ MIT
	\and Vitaly Feldman \\ Google Research
}
\title{Interaction is Necessary for Distributed Learning with Privacy or Communication Constraints}
\else
\author{Yuval Dagan}
\affiliation{%
\institution{EECS, MIT\footnotemark[1]\thanks{$*$ Part of the work was done while the author was at Google Research.}}
 \city{Cambridge}
 \state{MA}
 \country{USA}
}
\email{dagan@mit.edu}
\author{Vitaly Feldman}
\affiliation{%
\institution{Google Research\footnotemark[2]\thanks{$\dagger$ Now at Apple, Cupertino, CA, USA}}
 \city{Mountain View}
 \state{CA}
 \country{USA}
}
\email{vitaly.edu@gmail.com}

\title[Interaction Is Necessary for Distributed Learning with Privacy or Communication Constraints]{Interaction Is Necessary for Distributed Learning with Privacy or Communication Constraints}

\ccsdesc[500]{Theory of computation~Sample complexity and generalization bounds}
\ccsdesc[300]{Theory of computation~Communication complexity}
\ccsdesc[500]{Theory of computation~Boolean function learning}
\ccsdesc[300]{Theory of computation~Interactive computation}

\keywords{Local Differential Privacy, Distributed Learning, Interactive Protocol, Communication-constrained Learning, Statistical Query}
\fi

\begin{document}
\date{}

\iffull
\maketitle
\fi
\begin{abstract}
	Local differential privacy (LDP) is a model where users send privatized data to an untrusted central server whose goal it to solve some data analysis task. In the non-interactive version of this model the protocol consists of a single round in which a server sends requests to all users then receives their responses. This version is deployed in industry due to its practical advantages and has attracted significant research interest.

	Our main result is an exponential lower bound on the number of samples necessary to solve the standard task of learning a large-margin linear separator in the non-interactive LDP model. Via a standard reduction this lower bound implies an exponential lower bound for stochastic convex optimization and specifically, for learning linear models with a convex, Lipschitz and smooth loss. These results answer the questions posed by Smith, Thakurta, and Upadhyay (IEEE Symposium on Security and Privacy 2017) and Daniely and Feldman (NeurIPS 2019). Our lower bound relies on a new technique for constructing pairs of distributions with nearly matching moments but whose supports can be nearly separated by a large margin hyperplane. These lower bounds also hold in the model where communication from each user is limited and follow from a lower bound on learning using non-adaptive \emph{statistical queries}.
\end{abstract}

\iffull
\else
\maketitle
\fi

\section{Introduction}
The primary model we study is distributed learning with the constraint of local differential privacy (LDP) \citep{Warner,EvfimievskiGS03,KasiviswanathanLNRS11}. In this model each client (or user) holds an individual data point and a server can communicate with the clients. The goal of the server is to perform some statistical analysis on the data stored at the clients. In addition, the server is not trusted and the communication should not reveal significant private information about the users' data. Specifically, the entire protocol needs to satisfy differential privacy \citep{DworkMNS:06}. In the general version of the model, the executed protocol can involve an arbitrary number of rounds of interaction between the server and the clients.
In practice, however, network latencies significantly limit the number of rounds of interaction that can be executed. Indeed, currently deployed systems that use local differential privacy are non-interactive \citep{ErlingssonPK14,appledp,DKY17-Microsoft}. Namely, the server sends each client a request; based on the request each client runs some differentially private algorithm on its data and sends a response back to the server. The server then analyzes the data it received (without further communication with the clients).

This motivates the question: which problems can be solved by non-interactive LDP protocols?
This question was first formally addressed by \citetall{KasiviswanathanLNRS11} who also established an equivalence, up to polynomial factors, between algorithms in the statistical query (SQ) framework of \citet{Kearns:98} and LDP protocols\footnote{More formally, the equivalence is for a more restricted way to measure privacy based on composition of the privacy parameters of each message sent by a user.}. In this equivalence, non-interactive protocols correspond to non-adaptive SQ algorithms. Unfortunately, most SQ learning algorithms are adaptive and thus, for most problems, this equivalence only gives interactive LDP protocols. Using this equivalence, \citet{KasiviswanathanLNRS11} also constructed an artificial learning problem which requires an exponentially larger number of samples to solve by any non-interactive LDP protocol than it does when interaction is allowed.

Motivated by the industrial applications of the LDP model, \citetall{SmithTU17} studied the complexity of solving stochastic convex loss minimization problems by non-interactive LDP algorithms. In these problems we are given a family of loss functions $\{\ell(w;z)\}_{z\in Z}$ convex in $w$ and a convex body $\K \subseteq \R^d$.
For a distribution $P$ over $Z$ the goal is to find an approximate minimizer of
\[
\ell(w;P) := \E_{z \sim P} \ell(w; z).
\]
over $w\in \K$.
They gave a non-interactive LDP algorithm that uses an exponential in $d$ number of samples. Additionally, they showed that such dependence is unavoidable for the commonly used optimization algorithms whose queries rely solely on the information in the neighborhood of the query point $w$ (such as gradients or Hessians). Their bounds have been strengthened and generalized in a number of subsequent works \citep{DuchiRY18,WoodworthWSMS18,BalkanskiSinger18,DiakonikolasGuzman18,WangGX18,bubeck2019complexity} but the question of whether a non-interactive LDP protocol for optimizing convex functions with polynomial sample complexity exists remained open.

A recent work of \citet{DanielyF18} shows that there exist natural learning problems that are exponentially harder to solve by LDP protocols without interaction. Specifically, they consider PAC learning a class $C$ of Boolean functions over a domain $X$. A PAC learning algorithm for $C$ receives i.i.d.~samples $(x,f^*(x))$ where $x$ is drawn from an unknown distribution $D$ and $f^* \colon X \to \{-1,1\}$, and its goal is to find $\hat{f} \colon X \to \{-1,1\}$ which achieves a \emph{classification error} of at most $\alpha$, namely
\[
\err_{f^*,D}(\hat f) \doteq \Pr_{x\sim D}[f^\ast(x)\neq \hat f(x)] \le \alpha.
\]
\citet{DanielyF18} show that the number of samples required by any non-interactive LDP protocol to learn $C$ with a non-trivial error is lower bounded by a polynomial in the margin complexity of $C$. The margin of a linear separator $f$ over $S \subseteq \R^d$ captures how well the points $x$ with $f(x) = 1$ are separated from those with $f(x) = -1$, and is formally defined as
\begin{equation} \label{eq:def-margin}
\gamma(f, S) \doteq \sup_{ w\ne \mathbf{0}} \inf_{x \in S} f(x) \frac{ \langle x, w\rangle }{\|x\|_2\|w\|_2}.
\end{equation}
The margin complexity of $C$ is the inverse of the largest margin $\gamma$ that can be achieved by embedding $X$ into $\R^d$ such that every $f\in C$ can be realized as a linear separator with margin at least $\gamma$. It is a well-studied notion within learning theory and communication complexity, measuring the complexity of Boolean function classes and their corresponding sign matrices in  (e.g.~\citep{Novikoff:62,AizermanBR67,BoserGV92,ForsterSS01,Ben-DavidES02,Sherstov:08,LinialS:2009,KallweitSimon:11}). There exist known classes of functions, such as decision lists and general linear separators, that are PAC learnable by (interactive) SQ algorithms but have exponentially large margin complexity. Thus, non-interactive LDP protocols require an exponentially larger number of samples for PAC learning such classes than interactive ones. This result also leads to the question of whether all classes with inverse polynomial margin complexity can be learned efficiently non-interactively (see \citep{DanielyF19:open} for a more detailed discussion). Such large-margin linear classifiers are much more common in practice and are significantly easier to learn than general linear separators. For example, a simple Perceptron algorithm can be used instead of the more involved algorithms like the Ellipsoid method that are used when the margin is exponentially small.

\subsection{Our Results}

We show that both learning large-margin linear separators and learning of linear models with a convex loss require an exponential number of samples in the non-interactive LDP model.
Formally, we define the margin relative to a distribution on $\R^d$ as the margin relative to the support of the distribution: $\gamma(f, D) \doteq \gamma(f,\supp(D))$.
We give the following lower bound for learning large-margin linear classifiers.
\begin{theorem} \label{thm:informal-privacy}
Fix $\eps >0, \gamma \in (0,1/4]$, $r \in (0,1)$ and $d \ge \gamma^{-2-2r/5}$. Let $\A$ be a randomized, non-interactive $\epsilon$-LDP learning algorithm over $X = \{-1,1\}^d$ using $n$ samples. Assume that for any linear separator $f^*$ and distribution $D$ over $X$ with margin $\gamma(f^*,D) \ge \gamma$, $\A$ outputs a hypothesis $\hat f$ with an expected error of $\E_\A[\err_{f^*,D}(\hat f)] \le 1/2-\gamma^{1-r}$. Then, $n \ge \exp(C\gamma^{-2r/5})/e^{2\epsilon}$, where $C > 0$ depends only on $r$.
\end{theorem}
In particular, this lower bound is always exponential either in the margin or in the dimension of the problem. Note that linear separators with margin $\gamma$ can be learned with error $\alpha$ by an $\eps$-LDP algorithm with $O(1/\gamma^2)$ rounds of interaction and using $\poly(1/(\eps\alpha\gamma))$ samples. This can be done by using a standard SQ implementation of the Perceptron algorithm \citep{BlumFKV:97,FeldmanGV:15} (after a random projection to remove the dependence on the dimension) or via a reduction to convex loss minimization described below together with an LDP algorithm for convex optimization from \citep{DuchiJW:13focs}.
Our lower bound is also essentially tight in terms of the achievable error. There exist an efficient non-interactive algorithm achieving an error of $1/2-\Theta(\gamma)$, while $1/2-\gamma^{1-r}$ is impossible for all $r>0$. 
\vitaly{removed the claim of an upper bound since we don't have it.}

\paragraph{Proof technique:}
As in the prior work \citep{KasiviswanathanLNRS11,DanielyF18}, we exploit the connection to statistical query algorithms. Here, we assume a distribution $P$ over $Z = X\times Y$ and instead of i.i.d.~samples from $P$, an SQ algorithm has access to an SQ oracle for $P$. Given a query function $h\colon Z \to [-1,1]$ an SQ oracle for $P$ with tolerance parameter $\tau$ returns the value $\E_{z \sim P}[h(z)]$ with some added noise of magnitude bounded by $\tau$ \citep{Kearns:98}. Such an algorithm is non-adaptive if its queries do not depend on the answers to prior queries. Our lower bound is effectively a lower bound against non-adaptive statistical query algorithms together with the known simulation of a non-interactive LDP protocol by a non-adaptive SQ algorithm \citep{KasiviswanathanLNRS11}. The SQ model captures a broad class of learning algorithms and thus our lower bound can be viewed as showing the importance of interactive access to data beyond the distributed learning setting.

Our lower bound for non-adaptive SQ algorithms is based on a new technique for constructing hard to distinguish pairs of distributions over data. The key technical element of this construction is a pair of distributions over $\{-1,1\}^d$ that have nearly matching moments but whose supports are nearly linearly separable with significant margin. To design such distributions we rely on tools from the classical moment problem.

\paragraph{Convex loss optimization of linear models:}
We now spell out the implications of our lower bound in Theorem \ref{thm:informal-privacy} for stochastic convex optimization. Our lower bounds will apply to optimization of the simple class of \emph{convex linear models}. These models are defined by some loss function $\ell(w, (x,y)) = \varphi(\langle w,x\rangle, y)$ for some $\varphi$ that is convex in the first parameter for every $y$. In our reduction the label is in $\{-1,1\}$ and the loss function can be further simplified as $\ell(w; (x,y)) = \varphi(y \langle w,x\rangle )$ for a fixed convex function $\varphi \colon [-1,1] \to \mathbb{R}$. In our reduction $w$ and $x$ are in $B_d$, the unit ball of $\mathbb{R}^d$.
We show that there exists $L$-Lipschitz, $\sigma$-smooth and $\mu$-strongly convex $\varphi$ such that the following lower bound holds.
\begin{theorem}
\label{thm:convex-intro}
	For any parameters $0 \le \mu < \sigma \le \infty$, $L > 0$ and $\alpha > 0$, there exists a loss function $\ell(w,(x,y)) = \varphi(y \langle w,x\rangle)$ where $\varphi$ is convex, $L$-Lipschitz, $\sigma$-smooth and $\mu$-strongly convex, such that any non-interactive $\epsilon$-LDP algorithm $\A$ that outputs $\hat{w}$ satisfying $\E_\A[\ell(\hat{w},P)] \le \inf_{w \in B_d} \E[\ell(w,P)] + \alpha$, requires
	\[
	n \ge
	\min\left(
	\exp\left(c d^{0.16}\right),
	\exp\left(
	c \left(\frac{\min(L,\sigma)}{\max(\mu,\alpha)}
	\right)^{0.19}
	\right)
	\right),
	\]
	samples, where $c > 0$ is a universal constant.
\end{theorem}
This implies that with $1$-Lipschitzness and $1$-smoothness, the sample complexity is exponential either in $d$ or in $1/\alpha$, and if we add the assumption of $\mu$-strong convexity, the sample complexity can be exponential in $\kappa\doteq\sigma/\mu$. For comparison, for general convex functions the only known upper bounds are exponential in the dimension \citep{SmithTU17,WangGX18}. For linear models, by polynomial approximation it is possible to obtain bounds without an exponential dependence in the dimension: for example, \citet{ZhengML17collect} showed that logistic regression can be  solved with roughly $n=\alpha^{-O(\log\log(1/\alpha))}$ samples and \citet{WangSX19} study general linear models. \footnote{The bound stated by  \citet{WangSX19} is $n=\alpha^{-O(\log\log(1/\alpha))}$ for arbitrary 1-Lipschitz losses, contradicting the lower bound in Thm.~\ref{thm:convex-intro}. The authors have confirmed a mistake in their analysis and are working on correcting the bound \citep{SmithWang19pc}.}.
Efficient non-interactive LDP algorithms exist for least squares linear regression \citep{SmithTU17} and principal component analysis \cite{Wang2019principal} since for these tasks low order statistics suffice for finding a solution. 

\ifshort
\else
\paragraph{Communication constrained setting:}
An additional benefit of proving the lower bound via statistical queries is that we can extend our results to other models known to be related to statistical queries. In particular, we consider distributed protocols in which only a small number of bits is communicated from each client. Namely, each client applies a function with range $\{0,1\}^\ell$ to their input and sends the result to the server (for some $\ell \ll \log |Z|$). As the server only has to communicate a random seed which is practically small and can provably be compressed to $O(\log \log |Z| + \log n)$ bits, this model is useful when the communication cost is high and the complete sample $z \in Z$ is expensive to send, for example, when its dimension is large.
In the context of learning this model was introduced by \citet{Ben-DavidD98} and generalized by \citet{SteinhardtVW16}. Identical and closely related models are often studied in the context of distributed statistical estimation with communication constraints  (e.g.~\citep{luo2005universal,rajagopal2006universal,ribeiro2006bandwidth,ZhangDJW13,SteinhardtD15,suresh2016distributed,acharya2018inference, acharya2019hadamard, acharya2019distributed,acharya2019inference}).
As in the setting of LDP, the number of rounds of interaction that the server uses to solve a learning problem is a critical resource. Using the equivalence between this model and SQ learning we immediately obtain analogous lower bounds for this model. In particular, we show that either $\ell \ge \Omega(\gamma^{-0.39})$ or $n \ge \exp(\Omega(\gamma^{0.39}))$ is required for learning non-interactively. See Section~\ref{sec:low-comm} for additional details.
\fi
\paragraph{Future work:}
Our work provides nearly tight lower bounds for learning by non-interactive or one-round LDP protocols. An important question left open is whether linear classification and convex optimization can be solved by algorithms using a small number of rounds of interaction in the above models. Such lower bounds are not known even for the harder problem considered in \citep{DanielyF18}. In contrast, known techniques for solving these problems require a polynomial number of rounds (see \citep{SmithTU17} for a discussion).  We hope that the construction in this paper will provide a useful step toward lower bounds against multi-round SQ or LDP algorithms. We remark, however, that general multi-round LDP protocols can be stronger than statistical query algorithms \citep{JosephMR19} and thus may require an entirely different approach.

\remove{

local privacy and constrained communication using a standard reduction . As the above number of iterations is large, it would be interesting to try and close the gap between the bounds.  Still, we hope that the construction in this paper could be adapted to the more general multi-round LDP model.

Multi-round statistical query lower bounds could be interesting on their own, as they will give evidence for the hardness of solving machine learning algorithms non-iteratively: while all efficient algorithms known today for solving the above tasks require a large number of iterations\footnote{If the optimized function is well conditioned then it can be optimized with a few iterations, but the dependence in the condition number is still polynomial} and cannot be solved fast in parallel, no lower-bound limitation is known. Since most learning algorithms known today and basically all algorithms that are applied for practical learning tasks, can be implemented using statistical queries, lower bounds in this model will shed light on those limitations.

 Interestingly so, in an adaptive model, the best known round complexity for efficient non-adaptive algorithms is approximately $\tilde{O}(\poly(\min(d, 1/\alpha,\kappa)))$, obtained by adapting gradient based methods \citep{SmithTU17}.
}

\subsection{Related Work}
Most positive results for non-interactive LDP model concern relatively simple data analysis tasks, such as computing counts and histograms (e.g.~\citep{HsuKR12,ErlingssonPK14,BassilyS15,BunNS18,ErlingssonFMRTT18}). Efficient non-interactive algorithms for learning large-margin classifiers and convex linear models can be obtained given access to public unlabeled data \cite{DanielyF18,WangZGX2019}. A number of lower bounds on the sample complexity of LDP algorithms demonstrate that (non-interactive) LDP protocols are less efficient than the central model of differential privacy \citep{KasiviswanathanLNRS11,DuchiWJ13:nips,Ullman18,duchi2019lower}.

Joseph et al.~\citep{JosephMNR:19,JosephMR19} explore a different aspect of interactivity in LDP. Specifically, they distinguish between two types of interactive protocols: fully-interactive and sequentially-interactive ones. Fully-interactive protocols place no restrictions on interaction whereas sequentially-interactive ones only allows asking one query per user. They give a separation showing that sequentially-interactive protocols may require exponentially more samples than fully interactive ones. This separation is orthogonal to ours since our lower bounds are against completely non-interactive protocols and we separate them from sequentially-interactive protocols. \citet{acharya2018inference} implicitly consider another related model: \emph{one-way non-interactive protocols} where the server does not communicate the choice of a randomizer to the clients or, equivalently, cannot share a random string with clients. They give a polynomial separation between one-way non-interactive protocols and non-interactive protocols for the problem of identity testing for a discrete distribution over $k$ elements ($O(k)$ vs $\Omega(k^{3/2})$ samples).

Finally, we would like to add that moment matching was used in prior work to derive statistical query lower bounds for mixture distributions \cite{diakonikolas2017statistical,chen2019beyond}, to defend against adversarial examples~\cite{bubeck2019adversarial}, for robust statistics \cite{diakonikolas2019efficient}, for mean estimation with general norms \cite{li2019mean} and in other settings. Their proofs required different constructions from the one appearing in this paper.

\ifshort
\else
\section{Preliminaries}

\subsection{Models of Computation}
 \label{sec:prelim-models}
\paragraph{Local differential privacy:}
In the local differential privacy (LDP) model \citep{Warner,EvfimievskiGS03,KasiviswanathanLNRS11} it is assumed that each of $n$ users holds a sample of some dataset $(z_1,\dots,z_n) \in Z^n$. In the general version of the model the users can communicate with the server arbitrarily. The protocol is said to satisfy $(\eps,\delta)$-LDP if the algorithm that outputs the transcript\footnote{The transcript is the collection of all messages sent in the protocol.} of the protocol given the dataset $(z_1,\dots,z_n)$ satisfies the standard definition of $(\eps,\delta)$-differential privacy \citep{DworkMNS:06}.

We are interested in the non-interactive (one-round) LDP protocols. Such protocols can equivalently be described as non interactively accessing the following oracle:
\begin{definition}
An $\eps$-DP local randomizer $R:Z \rightarrow W$ is a randomized algorithm that given an input $z \in Z$, outputs a message $w \in W$, such that $\forall z_1,z_2\in Z$ and $w\in W$,
$\Pr[R(z_1) = w] \leq e^\eps \Pr[R(z_2) = w]$. For a dataset $S \in Z^n$, an $\LR_S$ oracle takes as an input an index $i$ and a local randomizer $R$ and outputs a random value $w$ obtained by applying $R(z_i)$. An algorithm is non-interactive $\eps$-LDP if it accesses $S$ only via the $\LR_S$ oracle with $\eps$-DP local randomizers, each sample is accessed at most once and all of its queries are determined before observing any of the oracle's responses.
\end{definition}

We remark that for non-interactive protocols, querying the same sample multiple times (subject to the entire communication satisfying $\eps$-DP) does not affect the model. Also for non-interactive protocols, allowing $(\eps,\delta)$-differential privacy instead of $\eps$-DP does not affect the power of the model \citep{BunNS18} (as long as $\delta$ is sufficiently small).

\paragraph{Statistical queries:}
The statistical query model of \citet{Kearns:98} is defined by having access to a \emph{statistical query oracle} to the data distribution $P$ instead of i.i.d.~samples from $P$. The oracle is defined as follows:

\begin{definition}
	Given a domain $Z$, a \emph{statistical query} is any (measurable) function $h \colon Z \to [-1,1]$. A \emph{statistical query oracle} $\STAT_P(\tau)$ with tolerance $\tau$ receives a statistical query $h$ and outputs an arbitrary value $v$ such that $|v -\E_{z\sim Z}[h(z)]| \leq \tau$.
\end{definition}
To solve a learning problem in this model an algorithm has to succeed for any oracle's responses that satisfy the guarantees on the tolerance. In other words, the guarantees of the algorithm should hold in the worst case over the responses of the oracle. A randomized learning algorithm needs to succeed for any SQ oracle whose responses may depend on the all queries asked so far but not on the internal randomness of the learning algorithm.

We say that an SQ algorithm is {\em non-interactive} (or {\em non-adaptive}) if all its queries are determined before observing any of the oracle's responses. \citet{KasiviswanathanLNRS11} show that one can simulate a non-interactive $\eps$-LDP algorithm using a non-adaptive SQ algorithm.
\begin{theorem}[\citep{KasiviswanathanLNRS11}]
\label{thm:LDP-2-SQ}
Let $\A$ be an $\eps$-LPD algorithm that makes non-interactive queries to $\LR_S$ for $S\in Z^n$ drawn i.i.d.~from some distribution $P$. Then for every $\delta >0$ there is a non-adaptive SQ algorithm $\A_{SQ}$ that in expectation makes $O(n \cdot e^\eps)$ queries to $\STAT_P(\tau)$ for $\tau =\Theta(\delta/(e^{2\eps}n))$ and whose output distribution has a total variation distance of at most $\delta$ from the output distribution of $\A$.

\end{theorem}

We remark that this simulation extends to interactive LDP protocols as long as they rely on local randomizers with the sum of privacy parameters used on every point being at most $\eps$. Such protocols, first defined in \citep{KasiviswanathanLNRS11} are referred to as compositional $\eps$-LDP. They are known to be exponentially weaker than the general interactive LDP protocols although the separation is known only for rather unnatural problems \citep{JosephMR19}. The converse of this connection is also known: SQ algorithms can be simulated by $\eps$-compositional LDP protocols (and this simulation preserves the number of rounds of interaction) \citep{KasiviswanathanLNRS11}.

\subsection{Boolean Fourier Analysis}

Boolean Fourier analysis concerns with the Fourier coefficients of functions of Boolean inputs, $h \colon \{-1,1\}^d \to \mathbb{R}$. Let $U_d$ be the uniform distribution over $\{-1,1\}^d$, and for any $S \subseteq [d]$, define the coefficient
\[
\hat{h}(S) = \E_{x\sim U_d}[h(x) \chi_S(x)],\quad
\text{where } \chi_S(x) = \prod_{i \in S} x_i.
\]
As $\{\chi_S(x)\}_{S \subseteq [d]}$ is an orthonormal basis of the space of functions $f \colon \{-1,1\}^d \to \mathbb{R}$, $h$ can be decomposed as $h(x) = \sum_{S \subseteq [d]} \hat{h}(S) \chi_S(x)$.
\emph{Plancherel's theorem} states that
\begin{equation} \label{eq:plancherel}
\E_{x \sim U_d}[h(x)g(x)]
= \sum_{S\subseteq [d]} \hat{h}(S)\hat{g}(S),
\end{equation}
and \emph{Parseval's theorem} is the special case where $g = h$.
For a distribution $D$ over $\{-1,1\}^d$ we define the Fourier coefficient as the coefficients of the function $x \mapsto \Pr_D[x]/\Pr_{U_d}[x]$, namely,
\begin{equation} \label{eq:def-Fourier-dist}
\hat{D}(S) = \E_{x\sim U_d}\left[
\frac{\Pr_D[x]}{\Pr_{U_d}[x]}
\chi_S(x)
\right]
= \E_{x \sim D} \chi_S[x].
\end{equation}
Lastly, note that for a distribution $D$ and a function $h$, it follows from Plancherel's theorem that
\begin{equation} \label{eq:apply-plancherel}
\E_{x \sim D}[h(x)]
= \E_{x \sim U_d}
	\left[h(x) \frac{\Pr_D[x]}{\Pr_{U_d}[x]}\right]
= \sum_{S\subseteq[d]} \hat{D}(S) \hat{h}(S).
\end{equation}

\subsection{The Classical Moment Problem} \label{sec:prel-moment}

Given a probability distribution $P$ and $k \in \mathbb{N}$, it is natural to try and characterize all distributions that have the same first $k$ moments as $P$, namely, distributions $D$ with $\E_{x\sim D}[x^i] = \E_{x\sim P}[x^i]$ for all $i \in [k]$. There is a great literature in this topic, e.g.~\citep{akhiezer1965classical, krein1977markov} (see \cite{benjamini2012k} for an application in computer science). The study uses the notion of orthogonal polynomials:

\begin{definition}
	Let $P$ be a probability distribution over $\mathbb{R}$ with all moments finite. We say that a sequence of polynomials $p_0,p_1, \dots, p_k, \dots$ are \emph{orthogonal with respect to $P$} if the satisfy the following:
	\begin{itemize}
		\item For all $m \ge 0$, $p_m$ is of degree $m$ and has a positive leading coefficient.
		\item For all $m, \ell \ge 0$, $\E_{x \sim P}[p_m(x) p_\ell(x)] = \mathds{1}_{m = \ell}$.
	\end{itemize}
	Denote the above sequence of polynomials as the \emph{orthogonal polynomials} with respect to $P$.
\end{definition}
It is known that there is a unique sequence of orthogonal polynomials with respect to $P$, hence we call them \emph{the orthogonal polynomials} (w.r.t $P$).
Given the orthogonal polynimials $p_0,p_1,\dots$, define the function $\rho_k \colon \mathbb{R} \to \mathbb{R}$ as follows:
\begin{equation} \label{eq:def-rho}
\rho_k(x) = \frac{1}{\sum_{i=0}^k p_i(x)^2}.
\end{equation}
These functions characterize the amount of mass that can be concentrated on the point $x$ by distributions $D$ that match the first $2k$ moments of $P$:
\begin{theorem}[\citep{akhiezer1965classical}, Theorem 2.5.2] \label{thm:moment-match}
	Let $P$ be a distribution with finite moments, fix $k \in \mathbb{N}$ and $x \in \mathbb{R}$ and let $\rho_k$ be defined with respect to $P$. The following holds:
	\begin{itemize}
		\item There exists a distribution $D$ matching the first $2k$ moments of $P$ with $\Pr_{D}[x] = \rho_k(x)$.
		\item Any distribution $D$ that matches the first $2k$ moments of $P$ satisfies: $\Pr_D[x] \le \rho_k(x)$.
	\end{itemize}
\end{theorem}

\section{Proof of Theorem \ref{thm:informal-privacy}}

Below we state and prove the lower bound on learning with statistical queries. The lower bounds for LDP protocols stated in Theorem \ref{thm:informal-privacy} follows directly from the reduction in Theorem \ref{thm:LDP-2-SQ}.

\begin{theorem} \label{thm:sq}
	Let $r \in (0,1)$, $\gamma \in (0,2^{-1/(1-r)})$, $n \ge \gamma^{-2-2r/5}$ and define $\eta = \gamma^{1-r}$. Let $\A$ be a non-adaptive statistical query algorithm such that for any linear separator $f^*$ and distribution $D$ over $X= \{-1,1\}^{2d}$ with margin $\gamma(f^*,D) \ge \gamma$, returns a hypothesis $\hat{f}$ with $\E_{\A}[\err_{f^*,D}(\hat{f})] \le 1/2-\eta$. If $\A$ has access to statistical queries with tolerance $\tau = \exp\left(-c\gamma^{-2r/5}\right)$, then $\A$ requires at least $\exp\left(c\gamma^{-2r/5}\right)$ queries, where $c>0$ is a constant depending only on $r$.
\end{theorem}

\subsection{Outline}
\label{sec:proof-outline}
We start with a brief sketch of the proof. Let $X = \{-1,1\}^{2d}$ and $Y = \{-1,1\}$. Our proof is based on a construction of two distribution $D_0$ and $D_1$ over $\{-1,1\}^{2d}\times \{-1,1\}$ and two linear functions $f_0$ and $f_1$ that are hard to distinguish but they almost always disagree on the label $y$. Specifically, the have the following properties:
\begin{itemize}
	\item Any $(x,y) \in \supp(D_b)$ satisfies $y = f_b(x)$ for $b \in \{0,1\}$, and additionally, $f_0$ and $f_1$ have $\Omega(\gamma)$-classification margin over the supports of $D_0$ and $D_1$, respectively.
	\item $D_0$ and $D_1$ have nearly the same Fourier coefficients: for any $S \subseteq [2d]$, $|\hat{D_0}(x) - \hat{D_1}(x)|$ is exponentially small.
	\item $f_0(x) \ne f_1(x)$ for nearly all values of $x$: $\Pr_{(x,y)\sim D_b}[f_0(x) = f_1(x)] = O(\eta)$, for $b \in \{0,1\}$ where $\eta := \gamma^{1-r}$.
\end{itemize}
Given these two distributions, we can create a hard family of distributions containing many pairs obtained from the original pair by a simple translation. Any efficient SQ algorithm would find most pairs of distributions impossible to distinguish. That is, the algorithm cannot distinguish which of the two distributions in the pair is the correct one. As a consequence, it will not be able to predict the correct label of $x$ for most values of $x$.

In the rest of this section we describe how $D_0$ and $D_1$ are constructed. The construction involves multiple consecutive steps that we describe below. We start with two distributions $P$ and $Q$ over $\mathbb{R}$ that satisfy:
\begin{enumerate}
	\item $P$ and $Q$ have matching first $2k = \gamma^{-\Omega(1)}$ moments.
	\item $\Pr_{p\sim P}[p \ge \gamma] = 1$ and $\Pr_{q \sim Q}[q\le-\gamma] \ge 1-O(\eta)$, where $\eta = \gamma^{1-r}$.
\end{enumerate}
The distribution $P$ is a mixture in which the value $\gamma$ has weight $1-\eta$ and a scaled and shifted exponential distribution defined on $[\gamma,\infty)$ has weight $\eta$.
To show that there exists a distribution $Q$ which matches the first $2k$ moments of $P$ and satisfies $\Pr_Q[-\gamma]\ge 1-O(\eta)$, it suffices to show that $\rho_k(-\gamma) \ge 1-O(\eta)$, where $\rho_k$ is the function from \eqref{eq:def-rho}, which is defined by the orthogonal polynomials of $P$. We calculate these polynomials as a linear combination of the orthogonal polynomials of the exponential distribution, for which a closed formula is known. We remark that instead of the exponential distribution other distributions can be used to get a similar bound on $\rho_k$.

Based on $P$ and $Q$, we create two distributions $P_1$ and $P_{-1}$ over $\{-1,1\}^d$ which satisfy:
\begin{itemize}
	\item $P_1$ and $P_{-1}$ nearly match all Fourier coefficients.
	\item $\Pr_{x\sim P_1}[\sum_i x_i/d \ge \gamma/2] =1$ and $\Pr_{x\sim P_{-1}}[\sum_i x_i/d \le -\gamma/2] \ge 1 - O(\eta)$.
\end{itemize}
To draw $x \sim P_1$ we first draw $p \sim P$ and then draw each bit of $x$ independently with mean $p$. Similarly, we draw $P_{-1}$ given $Q$. The Fourier coefficients of $P_1$ and $P_{-1}$ correspond to the moments of $P$ and $Q$, respectively: $\hat{P_1}(S) = \E_{P}[p^{|S|}]$ and similarly for $P_{-1}$ and $Q$. Hence the Fourier coefficients of $P_1$ and $P_{-1}$ nearly match (note that we've only shown that $P$ and $Q$ match the first $2k$ moments, however, the higher moments are exponentially small and negligible). The second property of $P_1$ and $P_{-1}$ follows from the second property of $P$ and $Q$ (except with some small failure probability which we can condition out).

Next, we explain the distributions $D_0$ and $D_1$ and the functions $f_0$ and $f_1$ that appear in the first paragraph: $f_0$ is defined as a majority over the first $d$ bits, $f_0(x) = \sgn(\sum_{i=1}^d x_i)$ and $f_1$ is a majority over the last $d$ bits, $f_1(x) = \sgn(\sum_{i=d+1}^{2d} x_i)$. To draw $(x,y)\sim D_0$, we independently draw $y \sim \unif(\{-1,1\})$, $z_1 \sim P_1$ and $z_{-1} \sim P_{-1}$. Then, we set $x = (yz_1, yz_{-1})$. We define $D_1$ nearly the same way, with the only difference that $x = (yz_{-1}, yz_1)$. From the properties of $P_1$ and $P_{-1}$, all properties of $D_0$ and $D_1$ presented in the first paragraph are satisfied.

\subsection{Proof of Theorem \ref{thm:sq}}
We begin with some notations:
\begin{itemize}
	\item Given a statistical query $h$, denote $h(D,f) = \E_{x\sim D}(h(x,f(x)))$.
	\item We use $C,C',c,c_1,\dots > 0$ to denote universal constants or constants depending only on $r$. In the proof we will allow redundant constants depending on $r$ (e.g.~the advantage will be $C \eta$ rather than $\eta)$.
	\item Let $\unif(A)$ denote the uniform distribution over a finite set $A$, let $\dtv$ denote the total variation distance of two distributions and let $\supp(P)$ denote the support of a probability distribution $P$.
	\item In contrast to the presentation in the intro, we conveniently assume that the distributions $D$ are only over $X$ rather than over $X \times Y$.
\end{itemize}

The general idea is to split the $2d$ bits of $x$ into two bit-sets, each containing $d$ bits. The value of $f^*(x)$ will be a function of one of these sets, however any efficient non-adaptive algorithm would not be capable of finding the correct subset. Moreover, intuitively speaking, the incorrect subset will almost always \emph{lie} by claiming the wrong value for $f^*(x)$.


We begin with two distributions $P_1$ and $P_{-1}$ that nearly match all Fourier coefficients, however, $\sgn(\sum_{i=1}^{2d} x_i) = 1$ for any $x \in \supp(P_1)$ while $\sgn(\sum_i x_i) = -1$ with probability $1- O(\eta)$ for $x \sim P_{-1}$.

\begin{lemma} \label{lem:indisting-dist}
	There exists two distributions, $P_1$ and $P_{-1}$ over $\{-1,1\}^d$, such that the following holds:
	\begin{enumerate}
		\item $\dtv(P_1, -P_{-1}) \le C\eta$, where $x \sim -P_{-1}$ is obtained by drawing $x \sim P_{-1}$ and outputting $-x$, and $C>0$ is a constant depending only on $r$. \label{itm:TV}
		\item Any $x \sim \supp(P_1)$ satisfies $\sum_i x_i/d \ge C\gamma$. \label{itm:margin}
		\item $P_1$ and $P_{-1}$ are nearly indistinguishable: for any $S \in [d]$, $|\widehat{P_1}(S) - \widehat{P_{-1}}(S)| \le \exp(-c\gamma^{-2r/5})$, where $c>0$ is a constant depending only on $r$. \label{itm:closeness}
	\end{enumerate}
\end{lemma}

The proof utilizes results from the classical moment problem, and involves calculating the orthogonal polynomials of some distribution,
as will be elaborated in Section~\ref{sec:pr-two-dist}.

Given $P_1$ and $P_{-1}$, we construct two pairs of distribution-function $(f_0,D_0)$ and $(f_1,D_1)$ which are hard to distinguish, in a sense that will be clear later.
The function $f_0$ is a majority of the first $d$ coordinates, $f_0(x) = \sgn(\sum_{i=1}^d x_i)$ and $f_1$ is a majority of the last $d$ bits, $f_1(x) = \sgn(\sum_{i=d+1}^{2d} x_i)$. A random $x \sim D_0$ is drawn by drawing independently $y \in \mathrm{Uniform}(\{-1,1\})$, $z_1 \sim P_1$, $z_{-1} \sim P_{-1}$ and setting $x = (yz_1, yz_{-1})$. Note that $f_0(x) = y$, where $y$ is the value drawn above. Similarly, $x\sim D_1$ is drawn similarly, with the following distinction: $x = (yz_{-1}, yz_1)$. Here, notice that $f_1(x) = y$.



Since $P_1$ is nearly distributed as $-P_{-1}$, with high probability over $x\sim D_0$, the majority of the first $d$ coordinates of $x$ is almost always the opposite of the majority of the last last $d$ coordinates (and similarly when $x\sim D_1$). In particular, if one does not know whether the true function $f^*$ equals $f_0$ or $f_1$, it is impossible to predict $f^*(x)$ given $x$ with probability significantly greater than a half.

Utilizing the fact that the building blocks of $D_0$ and $D_1$, namely $P_1$ and $P_{-1}$, nearly match their Fourier coefficients, we can generate a family of hard distributions by simple translations of $D_0$ and $D_1$:
for any $a \in \{-1,1\}^{2d}$ define the pairs $(f_{a,0},D_{a,0})$ and $(f_{a,1}, D_{a,1})$ as follows:
$f_{a,0} = \sgn(\sum_{i=1}^d a_i x_i)$ and $x\sim D_{a,0}$ is obtained by drawing $x' \sim D_0$ and setting $x_i = a_i x'_i$ for $i \in [2d]$. Similarly, $f_{a,1} = \sgn(\sum_{i=d+1}^{2d} a_i x_i)$ and $D_{a,1}$ is obtained by drawing $x' \sim D_1$ and setting $x_i = a_i x'_i$.
The following are simple properties of the defined distributions, which follow mainly from Lemma~\ref{lem:indisting-dist}, and are proved in Section~\ref{sec:pr-lem-D}

\begin{lemma} \label{lem:disc-char}
	Fix $a \in \{-1,1\}^{2d}$.
	Then, $D_{a,0}$ and $D_{a,1}$ satisfy the following properties:
	\begin{enumerate}
		\item $\dtv(D_{a,0}, D_{a,1}) \le 2\dtv(P_1,-P_{-1}) \le C\eta$ \label{itm:D-TV}
		\item $\gamma(f_{a,0},D_{a,0}) \ge C\gamma$ and $\gamma(f_{a,1},D_{a,1}) \ge C\gamma$ \label{itm:D-margin}
		\item $\Pr_{x \sim D_{a,0}}[f_{a,0}(x) = f_{a,1}(x)] \le C\eta$ and
		$\Pr_{x \sim D_{a,1}}[f_{a,0}(x) = f_{a,1}(x)] \le C\eta$ \label{itm:D-pr}
	\end{enumerate}
	where $C>0$ depends only on $r$ (recall that $\eta = \gamma^{1-r}$).
\end{lemma}

Next, we claim that for any set of $\exp(O(\gamma^{-2r/5}))$ statistical queries and for nearly all values of $a$, the queries will have nearly the same value for both $(f_{a,0},D_{a,0})$ and $(f_{a,1},D_{a,1})$. This follows from the fact that $P_1$ and $P_{-1}$ have all their Fourier coefficient close to each other.
\begin{lemma}\label{lem:many-statQ-hardness}
	Fix a set of statistical queries $h_1, \dots, h_k$ for $k \le \exp(c_1 \gamma^{-2r/5})$. Then,
	\begin{align*}
	&\Pr_{a \in \{-1,1\}^{2d}}\left[\exists i \in [k],\ |h_i(D_{a,0},f_{a,0}) - h_i(D_{a,1},f_{a,1})|\right.\\
	&\left.\ge \exp(-c_2 \gamma^{-2r/5})\right]
	\le \exp(-c_3 \gamma^{-2r/5}),
	\end{align*}
	where $c_1,c_2,c_3 >0$ depend only on $r$.
\end{lemma}

The proof will be presented in Section~\ref{sec:pr-lem-hardness}.
Next, we define the exact statistical query setting: define the number of allowed queries $k$ and tolerance $\tau$ to ensure that the algorithm cannot distinguish between $(f_{a,0},D_{a,0})$ and $(f_{a,1},D_{a,1})$: $k = \exp(c_1 \gamma^{-2a/5})$ and $\tau = \exp(-c_2 \gamma^{-2a/5})$, for the constants $c_1,c_2$ from Lemma~\ref{lem:many-statQ-hardness}.
We define the SQ oracle such that it gives the same answers to $(f_{a,0},D_{a,0})$ and $(f_{a,1},D_{a,1})$ for most $a$: given a statistical query $h$, it acts as follows:
\begin{itemize}
	\item If the true distribution-function pair is $(f_{a,0},D_{a,0})$ for some $a \in \{-1,1\}^{2d}$ then return the true value $h(f_{a,0},D_{a,0})$.
	\item If the pair is $(f_{a,1},D_{a,1})$ and $|h(f_{a,0},D_{a,0}) - h(f_{a,1},D_{a,1})| \le \tau$ then return $h(f_{a,0},D_{a,0})$.
	\item Otherwise return $h(f_{a,1},D_{a,1})$.
\end{itemize}
To conclude the proof, recall that
Lemma~\ref{lem:disc-char} states that for nearly all values of $x$, $f_{a,0}(x)\ne f_{a,1}(x)$. In particular, if one cannot distinguish between these two functions, then they cannot know the true classification of $x$. There some delicacy that should be taken care of: if the total variation distance between $D_{a,0}$ and $D_{a,1}$ was large, it would have been possible, given $x$, to guess whether it was drawn from $D_{a,0}$ or $D_{a,1}$ with a non-negligible success probability. However, Lemma~\ref{lem:disc-char} ensures that this is not the case. The formal proof is presented below:

\begin{proof}[Proof of Theorem~\ref{thm:sq}]
	We start by assuming that the algorithm is deterministic and then extend to randomized algorithms. From this assumption it follows that the statistical queries $h_1,\dots,h_k$ are deterministic as well.
	Fix $a$ such that the responses of the oracle to $(D_{a,0},D_{a,1})$ are the same as for $(D_{a,1},f_{a,1})$. From Lemma~\ref{lem:many-statQ-hardness} and from the definition of the oracle, nearly all $a$ are such. For these $a$, the algorithm has to learn some hypothesis without knowing if the true distribution-function pair is $(D_{a,0},f_{a,0})$ or $(D_{a,1},f_{a,1})$. Let $\A_{a,b}$ denote the learned hypothesis given $(f_{a,b},D_{a,b})$. For these hard values of $a$, $\A_{a,0} = \A_{a,1}$. Let $\eta' := C\eta$, where $C$ is the constant from Lemma~\ref{lem:disc-char}.
	Applying Lemma~\ref{lem:disc-char} multiple times, we obtain that for any such $a$:
	\begin{align}
	&\Pr_{x \sim D_{a,0}}[\A_{a,0}(x) = f_{a,0}(x)]
	+ \Pr_{x \sim D_{a,1}}[\A_{a,1}(x) = f_{a,1}(x)]\notag\\
	&=\Pr_{x \sim D_{a,0}}[\A_{a,0}(x) = f_{a,0}(x)]
	+ \Pr_{x \sim D_{a,1}}[\A_{a,0}(x) = f_{a,1}(x)]\notag\\
	&\le\Pr_{x \sim D_{a,0}}[\A_{a,0}(x) = f_{a,0}(x)]
	+ \Pr_{x \sim D_{a,1}}[\A_{a,0}(x) \ne f_{a,0}(x)] \\
	&+\Pr_{x \sim D_{a,1}}[f_{a,1} = f_{a,0}(x)]\notag\\
	&\le\Pr_{x \sim D_{a,0}}[\A_{a,0}(x) = f_{a,0}(x)]
	+ \Pr_{x \sim D_{a,1}}[\A_{a,0}(x) \ne f_{a,0}(x)]
	+\eta'\notag\\
	&\le\Pr_{x \sim D_{a,0}}[\A_{a,0}(x) = f_{a,0}(x)]
	+ \Pr_{x \sim D_{a,0}}[\A_{a,0}(x) \ne f_{a,0}(x)]
	+2\eta'\label{eq:24}\\
	&= 1 + 2\eta'.\notag
	\end{align}
	where \eqref{eq:24} follows from the fact that $\dtv(D_{a,0},D_{a,1}) \le \eta'$.
	From Lemma~\ref{lem:many-statQ-hardness}, the above holds for a $1-\exp(-c \gamma^{2r/5})$-fraction of the values of $a$ (where $c>0$ depends only on $r$). In particular,
	\begin{equation} \label{eq:25}
	\Pr_{\substack{a\sim \unif(\{-1,1\}^{2d})\\b\sim \unif(\{0,1\}), x\sim D_{a,b}}}[\A_{a,b}(x) = f_{a,b}(x)] \le 1/2 + C' \eta',
	\end{equation}
	where $C'$ depends only on $r$.
	Lastly, assume that the algorithm is randomized. Any randomized algorithm is just a distribution over deterministic algorithms, hence \eqref{eq:25} will hold even if the algorithm is allowed to be randomized and the probability is taken over $a,b,x$ and the randomness of the algorithm.
\end{proof}

\subsection{Proof of Lemma~\ref{lem:indisting-dist}} \label{sec:pr-two-dist}

\paragraph{Notation.}
Throughout the proof we will use the following parameters: $\eta = \gamma^{1-r}$, $\gamma' = \gamma^{1-2r/5}$ and $k = \lfloor(\eta/\gamma')^{2/3}\rfloor = \lfloor\gamma^{-2r/5}\rfloor$. From the assumptions in Theorem~\ref{thm:sq}, $\gamma, \eta, \gamma' \in (0,1/2]$.

\paragraph{Outline.}
The first step is to find two distributions over $\mathbb{R}$ of a particular shape that their first $k$ moments match. The first distribution $P$ is a mixture that samples $0$ with probability $1-\eta$ and an exponential random variable with probability $\eta$. By calculating the orthogonal polynomials of $P$ and applying Theorem~\ref{thm:moment-match}, we find a distribution $Q$ that matches the first $k$ moments of $P$, and additionally, $\Pr_Q[-\gamma'] \ge 1 - O(\eta)$.

In the second step, we shift, scale and condition $P$ and $Q$, to obtain two distributions $P'$ and $Q'$ that have nearly matching moments and satisfy the following conditions: $\Pr_{P'}[\gamma] \ge 1-O(\eta)$; $\Pr_{Q'}[-\gamma] \ge 1 - O(\eta)$; $P'$ is supported on $[\gamma,1/2]$ and $Q'$ is supported on $[-1/2,1/2]$.

In the third step, we use $P'$ and $Q'$ to generate $P_1$ and $P_{-1}$, respectively. To generate $x = (x_1,\dots,x_d) \sim P_{-1}$, we first draw $p \sim Q'$ and then, conditioned on $p$, we draw each $x_i$ i.i.d. from the distribution over $\{-1,1\}$ with expectation $p$. The distribution $P_1$ is similarly defined using $P'$, except that we additionally condition on the high-probability event that $\sum_i x_i/d \ge \gamma/2$. It follows from a simple argument that the Fourier coefficients satisfy $\hat{P_{-1}}(S) = \E_{p \sim Q'}[p^{|S|}]$ and similarly, $\hat{P_1}(S) \approx \E_{p \sim P'}[p^{|S|}]$. We obtain that all Fourier coefficients of $P_1$ and $P_{-1}$ nearly match.

Lastly, we claim that $\dtv(P_1,-P_{-1}) \le O(\eta)$. To obtain this, first note that $\dtv(P',-Q') \le O(\eta)$, as both $P'$ and $-Q'$ have $1-O(\eta)$ mass on $\gamma$. As $P_1$ and $P_{-1}$ are obtained from $P'$ and $Q'$ using nearly the same transformation, we can apply the \emph{data processing inequality} to bound $\dtv(P_1,P_{-1}) \lesssim \dtv(P',Q')$.

We divide the proof into four parts, according to the steps described above.

\subsubsection*{Step 1: Distributions $P$ and $Q$ over $\mathbb{R}$ that match the first moments}
We start by constructing two distributions over $\mathbb{R}$ with matching first $2k$ moments. Let distribution $P$ be the following mixture: with probability $\eta$ sample from the exponential distribution with parameter $1$, and with probability $1-\eta$ sample $0$. We start with the following lemma:
\begin{lemma} \label{lem:match-moments-R}
	There exists a distribution $Q$ that matches the first $2k$ moments of $P$ and additionally, $\Pr_{x\sim Q}[x = -\gamma'] \ge 1 - C\eta$, where $C > 0$ is a universal constant.
\end{lemma}
Before proving this lemma, we give some intuition:
By Theorem~\ref{thm:moment-match}, it suffices to show that $\rho_k(-\gamma') \ge 1-O(\eta)$, where $\rho_k$ is as defined in Section~\ref{sec:prel-moment} with respect to the moments of $P$. The same theorem implies that since $\Pr_P[0] =1-\eta$, then $\rho_k(0) \ge 1-\eta$; and since $\rho_k$ is continuous, $\rho_k(-y) \ge 1-O(\eta)$ for any sufficiently small $y$. To show that $\rho_k(-\gamma') \ge 1 - O(\eta)$, we calculate the orthogonal polynomials of $P$ as linear combinations of the Laguerre polynomials, the orthogonal polynomials for the exponential distribution. Recall that $\rho_k$ is defined as a function of these polynomials, which allows us to bound $\rho_k$.

First, we present the orthogonal polynomials of the exponential distribution:
\begin{lemma}[\cite{spencer2015classical}, Chapter 3]
	The orthogonal polynomials for the exponential distribution with parameter $1$ are the Laguerre polynomials
	\[
	L_m(x) = \sum_{i=0}^m \binom{m}{i} \frac{(-1)^i}{i!} x^i.
	\]
\end{lemma}
Using a simple calculation, one obtains that the orthogonal polynomials $\{p_m\}_{m=0}^\infty$ for $P$ equal
\[
p_m(x) = \mu \left((m+\eta/(1-\eta))L_m -\sum_{\ell=0}^{m-1} L_\ell(x)\right),
\]
where
\[
\mu^{-2} = \eta(m+\eta/(1-\eta))^2 +  \eta m+ \eta^2/(1-\eta) = \eta (m^2+m) + O(m \eta^2).
\]
is a normalizing constant.
To verify this formula it suffices to check that $\E_P[p_m(x) L_\ell(x)] = 0$ for $\ell < m$ and that
$\E_P[p_m(x)^2] = 1$ and these equations uniquely define $p_m$ (up to sign changes).

To get a closed form equation of the orthogonal polynomial, we use the identity
\[
\sum_{\ell=i}^{m-1} \binom{\ell}{i} = \binom{m}{i+1},
\]
to obtain that
\begin{align*}
\sum_{\ell=0}^{m-1} L_\ell(x)
= \sum_{\ell=0}^{m-1} \sum_{i=0}^{\ell} \binom{\ell}{i} \frac{(-1)^i}{i!} x^i
= \sum_{i=0}^{m-1} \sum_{\ell=i}^{m-1} \binom{\ell}{i} \frac{(-1)^i}{i!} x^i
= \sum_{i=0}^{m-1} \binom{m}{i+1}\frac{(-1)^i}{i!} x^i,
\end{align*}
hence
\begin{equation} \label{eq:orth-compact}
p_m(x)/\mu = \sum_{i=0}^{m} \left((m+\eta/(1-\eta))\binom{m}{i} - \binom{m}{i+1}\right)\frac{(-1)^i}{i!} x^i,
\end{equation}
where $\binom{m}{m+1}=0$.

Using the above formula, we can prove the following bound on $\rho_k(x)$:
\begin{lemma} \label{lem:calc-rho}
	Assume that $|x| \le \eta k^{-3/2}$. Then, $\rho_k(x) \ge 1-C\eta$, where $C>0$ is a universal constant.
\end{lemma}
\begin{proof}
	We start by bounding the coefficients of $p_m(x)$. Denote $p_m(x) = \sum_{i=0}^m \xi_{m,i} x^i$.
	For any $i \le m$, we use that fact that
	\[
	\binom{m}{i+1} = \frac{m-i}{i+1}\binom{m}{i}
	\le (m+\eta/(1-\eta))\binom{m}{i}
	\]
	and \eqref{eq:orth-compact} to estimate
	\[
	|\xi_{m,i}| \le \mu (m+\eta/(1-\eta)) \binom{m}{i}/i!.
	\]
	Additionally, it follows from definition of $\mu$ that
	\begin{equation} \label{eq:mu-estimate}
	\mu \le \frac{1}{\sqrt{\eta}(m+\eta/(1-\eta))},
	\end{equation}
	hence
	\[
	|\xi_{m,i}| \le \binom{m}{i}\frac{1}{i! \sqrt{\eta}}.
	\]
	For $i = 0$ and $m > 0$ we can get a tighter bound using more accurate calculation, \eqref{eq:mu-estimate} and $\eta \le 1/2$:
	\[
	|\xi_{m,0}| = \frac{\mu \eta}{1-\eta}
	\le \frac{\sqrt{\eta}}{m(1-\eta)}
	\le \frac{2\sqrt{\eta}}{m}.
	\]
	
	We proceed with bounding $p_m$ for $m >0$, using the inequality $\binom{m}{i} \le m^i/i!$:
	\[
	|p_m(x)|
	\le \sum_{i=0}^m |\xi_{m,i}| x^i
	\le \frac{2\sqrt{\eta}}{m} + \sum_{i=1}^m \binom{m}{i} \frac{x^i}{i!\sqrt{\eta}}
	\le \frac{2\sqrt{\eta}}{m} + \sum_{i=1}^m \frac{(mx)^i}{(i!)^2 \sqrt{\eta}}.
	\]
	For any $1 \le m \le k$, by the requirement of this lemma, $mx \le kx \le \eta/\sqrt{k} \le 1$, hence
	\[
	|p_m(x)|
	\le \frac{2\sqrt{\eta}}{m} + \sum_{i=1}^\infty \frac{mx}{(i!)^2 \sqrt{\eta}}
	\le \frac{2\sqrt{\eta}}{m} + \frac{C mx}{\sqrt{\eta}}
	\le \frac{2\sqrt{\eta}}{m} + C\sqrt{\frac{\eta}{k}},
	\]
	where $C>0$ is a universal constant. Using the inequality $(a+b)^2 \le 2a^2 + 2b^2$ we obtain that
	\[
	\sum_{m=1}^k p_m(x)^2
	\le \sum_{m=1}^k \left(\frac{8\eta}{m^2} + \frac{2C^2\eta}{k}\right)
	= C' \eta,
	\]
	for a universal $C' > 0$.
	Since $p_0 \equiv 1$, we get that
	\[
	\rho_k(x) = \frac{1}{\sum_{m=0}^k p_m(x)^2} \ge \frac{1}{1+C'\eta} \ge 1 - C'\eta,
	\]
	as required.
\end{proof}
By definition, $\gamma' \le \eta/k^{3/2}$.
Combining Lemma~\ref{lem:calc-rho} and Theorem~\ref{thm:moment-match}, the proof of Lemma~\ref{lem:match-moments-R} concludes.

\subsubsection*{Step 2: Re-scaling and restricting $P$ and $Q$ to obtain $P'$ and $Q'$}
For any $\alpha,\beta \in \mathbb{R}$, let $\alpha P+\beta$
denote the distribution obtained in the obvious manner, by drawing $x \sim P$ and outputting $\alpha x + \beta$. In the same fashion, let
$P'$ denote the distribution $(P+\gamma'/2)/(8k+1)$ conditioned on $[-1/2,1/2]$ and $Q'$ denote $(Q+\gamma'/2)/(8k+1)$ conditioned on $[-1/2,1/2]$. Let $\tg = \gamma'/(16k+2)$ and note that $\tg = \Theta(\gamma)$. The following holds with respect to $P'$ and $Q'$:
\begin{lemma} \label{lem:condition-on-half}
	The following holds:
	\begin{itemize}
		\item $P'(\tg)\ge 1-C\eta$ and $Q'(-\tg) \ge 1-C\eta$.
		\item For any integer $i \ge 0$, $|E_{P'}[x^i] - \E_{Q'}[x^i]| \le e^{-ck}$.
	\end{itemize}
	(where $c,C>0$ are universal constants.)
\end{lemma}
Before proceeding with the proof, here is an intuition: the first item follows from the definitions of $P$, $P'$ and $Q'$ and Lemma~\ref{lem:match-moments-R}. For the second second item,
note that the first $k$ moments of $P'$ and $Q'$ nearly match because $P$ and $Q$ match these moments, and the remaining moments nearly match since they are small, since $P'$ and $Q'$ are supported on $[-1/2,1/2]$. In the proof we argue that conditioning on $[-1/2,1/2]$ does not matter much, by obtaining tail bounds on $P$ and $Q$, using a generalized Markov's inequality based on their first $2k$ moments.

\begin{proof}[Proof of Lemma~\ref{lem:condition-on-half}]
	The first item follows from $P(0) = 1-\eta$, Lemma~\ref{lem:match-moments-R} which states that $Q(-\gamma')\ge 1-C\eta$, and from the definitions of $P',Q'$ and $\tg$.
	
	Next, we prove the second item of the lemma. First, it is an easy exercise to check that $(P+\gamma'/2)/(8k+1)$ and $(Q+\gamma'/2)/(8k+1)$ match the first $k$ moments, as $P$ and $Q$ do.
	Next, we argue that conditioning on $[-1/2,1/2]$ does not change the first $k$ moments considerably, which would imply that $P'$ and $Q'$ nearly match those moments. This is obtained by bounding the tails of $(P+\gamma'/2)/(8k+1)$ and $(Q+\gamma'/2)/(8k+1)$. For that purpose, note that moment $m$ of the exponential distribution equals $m!$, hence,
	\begin{equation}\label{eq:236}
	\E_P[x^{2k}]
	= \eta (2k)!
	\le (2k)^{2k}.
	\end{equation}
	Using $\gamma' \le 1$, Markov's inequality, the fact that $P$ and $Q$ match the first $2k$ moments and \eqref{eq:236}, we obtain that for any $t \ge 1/2$,
	\begin{align}
	&\Pr_{(Q+\gamma'/2)/(8k+1)}\left[|x| \ge t\right]
	\le \Pr_{Q}\left[|x| \ge t(8k+1)-\gamma'/2\right]\notag\\
	&\le \Pr_Q\left[|x| \ge 8kt\right]
	= \Pr_Q\left[x^{2k} \ge (8kt)^{2k}\right]\notag\\
	&\le \E_Q\left[x^{2k}\right]/(8kt)^{2k}
	= \E_P\left[x^{2k}\right]/(8kt)^{2k}
	\le (4t)^{-2k}. \label{eq:237}
	\end{align}
	Using \eqref{eq:237} it is simple to see that for any moment $m \in [k]$,
	\begin{equation} \label{eq:44}
	\left|\E_{(Q+\gamma'/2)/(8k+1)}[x^m] - E_{Q'}[x^m]\right| \le e^{-ck}
	\end{equation}
	for some universal constant $c > 0$. From definition of $P$, it is also easy to see that
	\begin{equation} \label{eq:45}
	\left|\E_{(P+\gamma'/2)/(8k+1)}[x^m] - E_{P'}[x^m]\right| \le e^{-ck}.
	\end{equation}
	\eqref{eq:44}, \eqref{eq:45} and the fact that $(Q+\gamma'/2)/(8k+1)$ and $(P+\gamma'/2)/(8k+1)$ match their first $2k$ moments, imply that $|\E_{Q'}[x^m] - \E_{P'}[x^m]| \le 2e^{-ck}$ for any $m \in [k]$.
	
	Lastly, it remains to argue that $P'$ and $Q'$ nearly match the moments $m > k$. Indeed,
	\[
	|\E_{Q'}[x^m] - \E_{P'}[x^m]| \le
	|\E_{Q'}[x^m]| + |\E_{P'}[x^m]|
	\le 2^{-m} + 2^{-m}
	\le 2^{-k+1},
	\]
	using the fact that $P'$ and $Q'$ are supported on $[-1/2,1/2]$.
\end{proof}

\subsubsection*{Step 3: moving from $\mathbb{R}$ to the Boolean cube}
Using the distributions $P'$ and $Q'$ we define distributions $P_B,Q_B$ over the boolean cube $\{-1,1\}^d$, 
where $x \sim P_B$ is drawn as follows:
first, we draw $p \sim P'$. Conditioned on $p$, each bit $x_i$ is drawn independently such that $\E[x_i \mid p] = p$. Equivalently, $\Pr[x_i = 1] - \Pr[x_i=-1] = p$. Similarly, $Q_B$ is defined when $P'$ is replaced with $Q'$.
We obtain that for any set $S \subseteq [d]$,
\begin{align*}
\hat{P_B}(S)
=\E_{x \sim P_B}\left[\prod_{i\in S} x_i\right]
= \E_{p \sim P}\left[\E\left[\prod_{i\in S} x_i ~\middle|~ p\right]\right]
= \E_{p \sim P}\left[\prod_{i \in S} \E\left[x_i ~\middle|~ p\right]\right]
= \E_{p\sim P}\left[p^{|S|}\right],
\end{align*}
and similarly, $\hat{Q_B}(S) = \E_{Q}[p^{|S|}]$.
Hence, Lemma~\ref{lem:condition-on-half} implies that
$|\hat{P_B}(S) - \hat{Q_B}(S)| \le e^{-ck}$
for any $S \subseteq [d]$.

Notice that $P_B$ and $Q_B$ almost satisfy the requirements of Lemma~\ref{lem:indisting-dist} as $P_1$ and $P_{-1}$, however, it is required that $\sgn(\sum_i x_i) \ge \Omega(\gamma)$ for any $x \in \supp(P_1)$. Hence, we define the distribution $P'_B$ which equals $P_B$ conditioned on $\sum_i x_i/d \ge \tg/2 = \Omega(\gamma)$. The conditioning does not change the distribution considerably: since $p\sim P'$ always satisfies $p \ge \tg$, we obtain by Chernoff's inequality that
\begin{align} \label{eq:PB-concentrated}
\Pr_{x \sim P_B}\left[\frac{1}{d}\sum_i x_i \le \tg/2\right]
&= \E_{p\sim P}\left[ \Pr\left[\frac{1}{d}\sum_i x_i \le \tg/2 ~\middle|~ p\right] \right]\\
&\le e^{- d \tg^2/8}
\le e^{-kc},\notag
\end{align}
for a universal constant $c >0$, using the assumption $d \ge \Omega(k/\gamma^2) \ge \Omega (k/\tg^2)$. In particular, $|\hat{P_B}(S) - \hat{P'_B}(S)| \le e^{-ck}$ for any $S$, which implies, by the triangle inequality, that
\[
|\hat{Q_B}(S) - \hat{P'_B}(S)| \le
|\hat{Q_B}(S) - \hat{P_B}(S)|
+ |\hat{P_B}(S) - \hat{P'_B}(S)|
\le e^{-ck}
\]
for some other $c > 0$.
We set $P_1 = P'_B$ and $P_{-1} = Q_B$, and we have shown that these distributions satisfy statements \ref{itm:margin} and \ref{itm:closeness} of Lemma~\ref{lem:indisting-dist}.

\subsubsection*{Step 4: bounding the total variation between $P_1$ and $-P_{-1}$}

It remains to bound the total variation between $P_1$ and $-P_{-1}$.
Distributions $P_B$ and $Q_B$ are obtained from $P'$ and $Q'$ using the same transformation and thus, by the data processing inequality, 
we obtain that
\[
\dtv(P_B,-Q_B)
\le \dtv(P',-Q')
\le C \eta
\]
for some universal $C>0$, where the last inequality follows from the fact that $P'$ and $-Q'$ both have $1 - O(\eta)$ mass on $\tg$.
From \eqref{eq:PB-concentrated} and the definition of $P_B'$ it follows that $\dtv(P_B,P_B') \le e^{-ck}$, hence we get by the triangle inequality that
\begin{align*}
\dtv(P_1,-P_{-1})
= \dtv(P_B',-Q_B)
\le \dtv(P_B',P_B) + \dtv(P_B,-Q_B)
\le C \eta + e^{-ck}
\le C'(r)\eta,
\end{align*}
where $C'(r)$ is a constant that depends only on $r$, and the last inequality follows from the fact that $k$ and $1/\eta$ are polynomially related for a fixed $r$, hence $e^{-ck} \le C'(r)\eta$.
	
\subsection{Proof of Lemma~\ref{lem:disc-char}}\label{sec:pr-lem-D}

We start with proving statement~\ref{itm:D-TV} of the lemma. First, for any four probability distributions $P,P',Q,Q'$ defined on the same probability space, $\dtv(P\times Q, P'\times Q') \le \dtv(P,P') + \dtv(Q,Q')$, hence $\dtv(P_1 \times P_{-1}, (-P_{-1}) \times (-P_1)) \le 2\dtv(P_1,-P_{-1})$. Next, note that $D_0$ is obtained by drawing $x \sim P_1 \times P_{-1}$, drawing $y \sim \unif(\{-1,1\})$ and outputting $yx$, and $D_1$ is obtained from $(-P_{-1}) \times (-P_1)$ the same way. Hence, from the data processing inequality ,
\[
\dtv(D_0,D_1) \le \dtv(P_1 \times P_{-1}, (-P_{-1}) \times (-P_1)) \le 2\dtv(P_1,-P_{-1}).
\]
Next, note that $D_{a,0}$ is obtained from $D_0$ the same way that $D_{a,1}$ is obtained from $D_1$, hence, by the data processing inequality,
\[
\dtv(D_{a,0},D_{a,1})
\le \dtv(D_0,D_1)
\le 2\dtv(P_1,-P_{-1})
\le C\eta.
\]

For statement~\ref{itm:D-margin}, let $a' = (a_1,\dots,a_d,0,\dots,0)$. Then,
\begin{align*}
&\gamma(f_{a,0},D_{a,0})
\ge \inf_{x \in \supp(D_{a,0})} f_{a,0}(x) \frac{x^\top a'}{\|x\| \|a'\|}\\
&= \inf_{x \in \supp(D_0)} f_{a,0}(x) \frac{(x_1 a_1, \dots, x_{2d} a_{2d})^\top a'}{\|x\| \|a'\|} \\
&= \inf_{x \in \supp(D_0)} f_{a,0}(x) \frac{\sum_{i=1}^d x_i}{\|x\| \|a'\|}
=\inf_{x \in \supp(P_1)} \frac{\left|\sum_{i=1}^d x_i\right|}{\|x\| \|a'\|}\\
&= \inf_{x \in \supp(P_1)} \frac{\left|\sum_{i=1}^d x_i\right|}{\sqrt{2}d}
\ge \frac{C \gamma}{\sqrt{2}},
\end{align*}
where we used Lemma~\ref{lem:indisting-dist} for the last inequality. Similarly, we can lower bound $\gamma(f_{a,1},D_{a,1})$, using $a' = (0,\dots,0,a_{d+1},\dots,a_{2d})$.

Lastly, we prove statement~\ref{itm:D-pr}. To simplify notation, we will assume that $a = \mathbf{1}$ (the all-ones vector), however, it is simple to see that the statement holds for any $a$. Recall that $x \sim D_{\mathbf{1},0}$ is drawn by drawing $z_1 \sim P_1$, $z_{-1} \sim P_{-1}$, $y \sim \unif(\{-1,1\})$ and setting $x = (yz_1, yz_{-1})$. From section~\ref{itm:margin} of Lemma~\ref{lem:indisting-dist},
\begin{equation} \label{eq:f10}
f_{\mathbf{1},0}(x)
= \sgn\left(y\sum_{i = 1}^d (z_1)_i\right)
= y.
\end{equation}
From section~\ref{itm:TV} or Lemma~\ref{lem:indisting-dist}, $\dtv(P_1,-P_{-1}) \le C\eta$, which implies that with probability $1-C\eta$, $\sgn(\sum_{i=1}^d (z_{-1})_i) = -1$,
hence, with probability $1-C\eta$,
\begin{equation}\label{eq:f11}
f_{\mathbf{1},1}(x)
= \sgn\left(y\sum_{i = 1}^{d} (z_{-1})_i\right)
= -y.
\end{equation}
From \eqref{eq:f10} and \eqref{eq:f11}, the proof follows.
A similar statement holds when we replace $D_{\mathbf{1},0}$ with $D_{\mathbf{1},1}$, $D_{a,0}$ and $D_{a,1}$.
	
\subsection{Proof of Lemma~\ref{lem:many-statQ-hardness}} \label{sec:pr-lem-hardness}

We will be considering one statistical query.

\begin{lemma} \label{lem:one-statQ-hardness}
	Let $\theta := 2\max_{S \subseteq [d]} |\hat{P}_1(S) - \hat{P}_{-1}(S)|$. Fix a statistical query $h \colon \{-1,1\}^{2d} \times \{-1,1\}\to [-1,1]$. Then, for any $t > 0$,
	\[
	\Pr_{a \sim \unif(\{-1,1\}^{2d})}[|h(D_{a,0},f_{a,0}) - h(D_{a,1},f_{a,1})| \ge t]
	\le 2\theta^2/t^2.
	\]
\end{lemma}

Define the conditional distribution $D_{a,b\mid y=1}$ and as the conditional distribution of $x\sim D_{a,b}$ given $f_{a,b}(x)=1$ and similarly define $D_{a,b\mid y=-1}$. This enables us to decompose any statistical query $h$ in two: $h_1(x) = h(x,1)$ and $h_{-1}(x) = h(x,-1)$. Note that
\begin{align} \label{eq:21}
h(D_{a,b},f_{a,b})
&=\E_{x\sim D_{a,b}}[h(x,f_{a,b}(x))] \\
&= \Pr_{D_{a,b}}\left[f_{a,b}(x)=1\right]
\E_{D_{a,b\mid y=1}}\left[h_1(x)\right]\notag
+ \Pr_{D_{a,b}}\left[f_{a,b}(x)=-1\right]
\E_{D_{a,b\mid y=-1}}\left[h_{-1}(x)\right]\notag\\
&= \frac{1}{2} h_1(D_{a,b\mid y=1})
+ \frac{1}{2} h_{-1}(D_{a,b\mid y=-1}).\notag
\end{align}

Next, we present and prove a simple lemma:
\begin{lemma} \label{lem:fourier-tensorize}
	Let $P$, $P'$, $Q$ and $Q'$ be distributions over $\{-1,1\}^d$. Then, for any $S_1,S_2 \subseteq [d]$,
	\[
	|\widehat{P\times Q}(S_1,S_2) - \widehat{P' \times Q'}(S_1,S_2)|
	\le |\widehat{P}(S_1) - \widehat{Q}(S_1)| + |\widehat{P}(S_2) - \widehat{Q}(S_2)|.
	\]
\end{lemma}
\begin{proof}
	Note that
	\begin{align*}
	\widehat{P\times Q}(S_1,S_2)
	= \E_{x_1\sim P,x_2\sim Q} [\chi_{S_1,S_2}(x_1,x_2)]
	= \E_{x_1\sim P} [\chi_{S_1}(x_1)]
	\E_{x_2\sim Q} [\chi_{S_2}(x_2)]
	= \widehat{P}(S_1) \widehat{Q}(S_2).
	\end{align*}
	Hence,
	\begin{align*}
	&|\widehat{P\times Q}(S_1,S_2)
	- \widehat{P'\times Q'}(S_1,S_2)|
	= |\widehat{P}(S_1) \widehat{Q}(S_2)
	- \widehat{P'}(S_1)\widehat{Q'}(S_2)| \\
	&\qquad\le |\widehat{P}(S_1)| |\widehat{Q}(S_2) - \widehat{Q'}(S_2)|  +|\widehat{Q'}(S_2)| |\widehat{P}(S_1) - \widehat{P'}(S_1)|.
	\end{align*}
	From \eqref{eq:def-Fourier-dist}, each Fourier coefficient of a probability distribution is bounded by $1$ in absolute value, and the proof follows.
\end{proof}

As $|\hat{P_1}(S)-\hat{P_{-1}}(S)|\le \theta/2$ for all $S \subseteq [d]$, we obtain from Lemma~\ref{lem:fourier-tensorize} that:
\begin{align}\label{eq:33}
\forall S\subseteq [2d] \colon~
|\hat{D_{\mathbf{1},0\mid y=1}}(S) - \hat{D_{\mathbf{1},1\mid y=1}}(S)|
=|\hat{P_1 \times P_{-1}}(S)- \hat{P_{-1}\times P_1}(S)|
\le \theta,\notag
\end{align}
where $\mathbf{1}$ is the all-ones vector.
We use this inequality to prove the following lemma:


\begin{lemma} \label{lem:chebishev-usage}
	Fix $h_1 \colon \{-1,1\}^{2d} \to [-1,1]$ and $t > 0$. Then,
	\[
	\Pr_{a \sim \unif(\{-1,1\}^{2d})}[|h_1(D_{a,0\mid y=1}) - h_1(D_{a,1\mid y=1})| \ge t]
	\le \theta^2/t^2.
	\]
\end{lemma}

\begin{proof}
	Denote $P = D_{\mathbf{1},0\mid y=1}$ and $Q = D_{\mathbf{1},1\mid y=1}$. Denote $P_a = D_{a,0\mid y=1}$ and $Q_a = D_{a,1\mid y=1}$, and notice that $x\sim P_a$ is obtained by drawing $x'\sim P$ and setting $x_i = a_i x'_i$, and similarly $Q_a$ is obtained from $Q$. Hence $\hat{P_a}(S) = \hat{P}(S) \chi_S(a)$, where $\chi_S(a) = \prod_{i\in S} a_i$. Similarly, $\hat{Q_a}(S) = \hat{Q}(S) \chi_S(a)$.
	
	From \eqref{eq:apply-plancherel},
	\begin{align*}
	h_1(P_a) - h_1(Q_a)
	= \sum_S (\hat{P}_a(S)-\hat{Q}_a(S)) \hat{h_1}(S)
	= \sum_S (\hat{P}(S)-\hat{Q}(S)) \hat{h_1}(S) \chi_S(a).
	\end{align*}
	Squaring both sides, taking expectation over $a$, we obtain that
	\begin{align}
	&\E_a \left( h_1(P_a) - h_1(Q_a) \right)^2
	= \E_a \left(
	\sum_S \hat{h_1}(S) (\hat{P}_{a}(S) - \hat{Q}_{a}(S))
	\right)^2 \notag \\
	&\qquad=\sum_{S,T} \hat{h_1}(S) \hat{h_1}(T) (\hat{P}(S) - \hat{Q}(S))(\hat{P}(T) - \hat{Q}(T)) \E_{a} \chi_S(a) \chi_T(a) \notag\\
	&\qquad= \sum_{S} \hat{h_1}(S)^2 (\hat{P}(S) - \hat{Q}(S))^2 \label{eq:120}\\
	&\qquad\le \theta^2 \sum_S \hat{h_1}(S)^2 \label{eq:121}\\
	&\qquad= \theta^2 \E_{x \in \{-1,1\}^{2d}}[h(x)^2] \label{eq:122} \\
	&\qquad\le \theta^2. \label{eq:123}
	\end{align}
	where \eqref{eq:120} follows from $\E_a \chi_S(a) \chi_T(a) = \mathds{1}_{S = T}$, \eqref{eq:121} follows from the definitions of $P$ and $Q$,
	\eqref{eq:122} is Parseval's equality (\eqref{eq:plancherel}) and \eqref{eq:123} is due to the fact that by definition, $h_1(x) \in [-1,1]$ for all $x$.
	Therefore, by Chebyshev's inequality,
	\begin{align*}
	&\Pr_a[h_1(D_{a,0\mid y=1}) - h_1(D_{a,1\mid y=1})| \ge t]\\
	&=\Pr_a[|h_1(P_a) - h_1(Q_a)| > t] \le \mathrm{Var}(h_1(P_a) - h_1(Q_a)) / t^2 \le \theta^2/t^2.
	\end{align*}
\end{proof}

We are ready to conclude the proof of Lemma~\ref{lem:one-statQ-hardness}:
\begin{proof}[Proof of Lemma~\ref{lem:one-statQ-hardness}]
	From \eqref{eq:21},
	\begin{align}
	&\Pr_a[|h(D_{a,0},f_{a,0}) - h(D_{a,1},f_{a,1})| \ge t] \notag \\
	&\le \Pr_a[|h_1(D_{a,0\mid y=1}) - h_1(D_{a,1|y=1})| \ge t] \notag\\
	&+\Pr_a[|h_{-1}(D_{a,0\mid y=-1}) - h_{-1}(D_{a,1|y=-1})| \ge t] \notag\\
	&= \Pr_a[|h_1(D_{a,0\mid y=1}) - h_1(D_{a,1|y=1})| \ge t] \notag\\
	&+ \Pr_a[|h_{-1}(D_{-a,0\mid y=1}) - h_{-1}(D_{-a,1|y=1})| \ge t] \label{eq:130}\\
	&\le 2\theta^2/t^2, \label{eq:131}
	\end{align}
	where \eqref{eq:130} follows from $h_{-1}(D_{a,b|y=1}) = h_{-1}(D_{-a,b}|y=1)$ and \eqref{eq:131}
	follows from Lemma~\ref{lem:chebishev-usage}.
\end{proof}
Lastly, we conclude the proof of Lemma~\ref{lem:many-statQ-hardness}:
Lemma~\ref{lem:indisting-dist} implies that $\theta\le \exp(-c\gamma^{-2r/5})$, where $c>0$ is a constant depending only on $r$. Applying Lemma~\ref{lem:one-statQ-hardness} and taking union bound over $k = \exp(c_1\gamma^{-2r/5})$ statistical queries, the proof follows. 
\section{Lower bounds on convex optimization} \label{sec:opt}
In this section, we describe the implications of our main lower bound learning of linear models with a convex loss. 
Consider the task of optimizing a convex function $\ell(w;P) := \E_{(x,y) \sim P}\ell(w;(x,y))$, where $\ell$ is a convex and Lipschitz \emph{linear model}, namely,
$\ell(w;(x,y)) = \varphi(\langle w,x\rangle y)$ for a function $\varphi \colon [-1,1]\to \R$ which is convex and Lipschitz and $w$ is optimized over $B_d$, the unit ball in $\R^d$. Additionally, we will let $X = \mathbb{S}^{d-1}$ be the unit sphere in $\R^d$.
We present two reductions: first, the standard reduction to \emph{hinge loss} which is used in the {\em soft-margin support vector machine} (SVM) algorithm, and secondly, a reduction to a different function which is smooth and strongly convex, to show that a lower bound holds even given these assumptions.

In the first reduction, we set $\ell(w;(x,y)) = \max(0, \gamma - \langle w,x\rangle y)$.
This loss has the nice property that if $w \in \mathbb{S}^{d-1}$ classifies $(x,y)$ correctly with margin $\gamma$, namely, if $\langle w,x\rangle y \ge \gamma$, then $\ell(w;(x,y)) = 0$. At the same time, if $w$ misclassifies $(x,y)$, namely, $\langle w,x\rangle y \le 0$, then $\ell(w;(x,y)) \ge \gamma$ (essentially, $\ell$ can be viewed as a scaled \emph{surrogate loss function}).
Thus, if the distribution $P$ is linearly separable with margin $\gamma$, then $\min_{w \in B_d} \ell(w;P) = 0$, and any $w$ satisfying $\ell(w;P) = \gamma/3$ is an approximate linear separator: $\Pr_{(x,y)\sim P}[\sgn(\langle w,x\rangle) \ne y] \le 1/3$. Hence, we can reduce solving linear models to classification, obtaining the following result:

\begin{theorem} \label{thm:lb-formal}
	For any $\alpha > 0$ there exists a loss function $\ell(w,(x,y)) = \varphi(y \langle w,x\rangle)$ where $\varphi$ is convex and $1$-Lipschitz, such that any non-interactive $\epsilon$-LDP algorithm $\A$ that outputs $\hat{w}$ satisfying $\E_\A[\ell(\hat{w};P)] \le \inf_{w \in B_d} \ell(w;P) + \alpha$, requires
	\[
	n \ge \min\left(\exp\left(\alpha^{-\Omega(1)}\right), \exp\left(d^{\Omega(1)}\right)\right)/e^{2\epsilon}.
	\]
\end{theorem}
\begin{proof}
	Fix such algorithm $\A$. As discussed above, by simulating $\A$ we can approximately solve any classification problem with margin of at least $3 \alpha$. Denote $\gamma = \max(3\alpha, d^{-1/(2+1/5)})$.
	Applying the lower bound on classification (Theorem~\ref{thm:informal-privacy}), we derive that $n \ge \exp(\gamma^{-\Omega(1)})$, and the result follows.
\end{proof}

Before proceeding, we define smoothness and strong convexity:
\begin{definition}
	A differentiable function $f \colon [-1,1] \to \mathbb{R}$ with derivative $f'(x)$ is \emph{$\sigma$-smooth} and $\mu$-strongly convex for $0\le \mu\le\sigma$, if for any $x, x' \in [-1,1]$,
	\[
	f'(x)(y-x) + \frac{\mu}{2}(y-x)^2
	\le f(y) - f(x)
	\le f'(x) (y-x) +\frac{\sigma}{2} (y-x)^2.
	\]
\end{definition}
Next, we define the following convex loss function $\varphi_\gamma\colon [-1,1] \to\mathbb{R}$:
\[
\varphi_\gamma(t) = \frac{(1-t)^2}{8} + \begin{cases}
	1-2t/\gamma & -1 \le t \le 0\\
	(t-\gamma)^2/\gamma^2 & 0 \le t \le \gamma\\
	0 & \gamma \le t \le 1.
\end{cases}
\]
Note that $\varphi_{\gamma}(t) - (1-t)^2/8$ is non-negative, monotonic non-decreasing, $2/\gamma$-Lipschitz, $2/\gamma^2$-smooth and convex. This implies that $\varphi_{\gamma}(t)$ is non-negative, monotonic non-decreasing, $2/\gamma+1/2 \le 3/\gamma$-Lipschitz, $2/\gamma^2+1/4 \le 3/\gamma^2$-smooth and $1/4$-strongly convex.
Additionally, the following holds:
\begin{claim}\label{cla:simple-opt}
For any $t \ge \gamma$, it holds that $\varphi_\gamma(t) \le \varphi_\gamma(\gamma) \le 1/8$, while for any $t \le 0$, $\varphi_\gamma(t) \ge \varphi_\gamma(0) \ge 9/8$.
\end{claim}
Given a classification problem with distribution $P$ over $X \times Y$ and margin $\gamma$, we reduce it to the convex optimization problem with the loss function $\ell_\gamma(w;(x,y)) = \varphi_\gamma(\langle w,x\rangle y)$.
Claim~\ref{cla:simple-opt} and Markov's inequality imply the following connection between the classification error and the convex loss:
\begin{claim} \label{cla-simple-opt}
	For any vector $w \in B_d$,
	\[\
	\err_{P}(w)
	\le\frac{\ell_\gamma(w;P)}{9/8}
	\le \ell_\gamma(w;P),
	\]
	where $\err_{P}(w)$ is the classification error of the function $x \mapsto \sgn(\langle w,x\rangle)$.
	Additionally,
	$\inf_{w \in B_d} \ell_\gamma(w;P) \le 1/8$ (in particular, any unit-norm vector that classifies correctly with $\gamma$-margin would achieve this loss).
\end{claim}
Based on the above claim, we derive the following relationship between the expected error and expected loss:
\begin{claim} \label{cla:opt}
	Let $\A$ be a (randomized) algorithm that for some distribution $P$ over $\mathbb{S}^{d-1}\times \{-1,1\}$ outputs $\hat{w} \in B_d$ that satisfies $$\E_{\A}[\ell_\gamma(\hat{w};P)] \le \inf_{w \in B_d} \ell_\gamma(w;P) + 1/8. $$ Then, $\E_{\A}[\err_{P}(\hat{w})] \le 1/4$.
\end{claim}
\begin{proof}
	Applying both statements in Claim~\ref{cla-simple-opt}, we obtain that
	\[
		\E_{\A}[\err_{P}(\hat{w})]
		\le \E_{\A}[ \ell_\gamma(\hat{w};P)]
		\le \inf_{w \in B_d} \ell_\gamma(w;P) + 1/8
		\le 1/4.
	\]
\end{proof}
We are ready to state our lower bound for learning of a linear model with smooth and strongly convex loss:
\begin{theorem}
	For any parameters $0 \le \mu < \sigma \le \infty$, $L > 0$ and $\alpha > 0$, there exists a loss function $\ell(w,(x,y)) = \varphi(y \langle w,x\rangle)$ where $\varphi$ is convex, $L$-Lipschitz, $\sigma$-smooth and $\mu$-strongly convex, such that any non-interactive $\epsilon$-LDP algorithm $\A$ that outputs $\hat{w}$ satisfying $\E_\A[\ell(\hat{w};P)] \le \inf_{w \in B_d} \ell(w;P) + \alpha$, requires
	\begin{align*}
	n \ge \exp\left(
	c \min\left(
	\left(\frac{L}{\max(\mu,\alpha)}\right)^{2/5\cdot (1-\xi)}, \right.\right. 
	\left.\left.\left(\frac{\sigma}{\max(\mu,\alpha)}\right)^{1/5\cdot (1-\xi)},
	d^{1/6\cdot (1-\xi)}
	\right)
	\right),
	\end{align*}
	where $\xi$ can be any number in $(0,1)$ and $c > 0$ depends only on $\xi$.
\end{theorem}
\begin{proof}
	In the proof we will allow redundancy in the parameters up to universal constants (e.g., requiring $\ell$ to be $\Omega(\mu)$ rather than $\mu$ strongly convex). Assume that $\max(\mu,\alpha) \le \min(L,\sigma)$, otherwise the bound trivially follows.
	Denote $\theta = \max(\mu,\alpha)$ and let
	\begin{align*}
	\gamma 
	:=\max\left(
	\frac{\max(\mu,\alpha)}{L},
	\left(\frac{\max(\mu,\alpha)}{\sigma}\right)^{1/2},
	d^{-1/(2+2/5)}
	\right) 
	= \max\left(
	\frac{\theta}{L},
	\left(\frac{\theta}{\sigma}\right)^{1/2},
	d^{-1/(2+2/5)}
	\right).
	\end{align*}
	Consider the function $\theta \varphi_\gamma$: its Lipschitz constant is bounded by $3\theta/\gamma\le 3L$, the smoothness parameter is bounded by $3\theta /\gamma^2 \le 3\theta / (\sigma/\theta) = 3\sigma$ and the strong convexity parameter equals $\theta/4 \ge \mu/4$. Let $\A$ be an algorithm which finds an $\alpha/8$-optimal solution to the linear model defined by the function $\theta \varphi_\gamma$ (in expectation). Then, $\A$ finds an $\alpha/(8\theta)\le 1/8$-optimal solution to optimization of linear models with loss defined by $\varphi_\gamma$. By Claim~\ref{cla:opt}, $\A$ finds approximate linear separator to any classification problem with margin $\gamma$. Since $\gamma$ is defined to satisfy $d\ge \gamma^{-2-2/5}$, by the lower bound on classification (Theorem~\ref{thm:informal-privacy}), for any $r \in (0,1)$, the sample complexity satisfies
	\[
	n \ge \exp(c(\gamma^{-2r/5})),
	\]
	where the constant $c>0$ may depend only on $r$. Taking $r = 1-\xi$, completes the proof.
\end{proof} 
\section{Implications for distributed learning with communication constraints}
\label{sec:low-comm}
In this section we briefly define the model of bounded communication per sample, state the known equivalence results to the SQ model and spell out the immediate corollary of our lower bound. In the bounded communication model \citep{Ben-DavidD98,SteinhardtVW16} it is assumed that the total number of bits learned by the server about each data sample is bounded by $\ell$ for some $\ell \ll \log |Z|$. As in the case of LDP this is modeled by using an appropriate oracle for accessing the dataset. For simplicity we only introduce the non-interactive version of this model.
\begin{definition}
\label{def:comm}
We say that a (possibly randomized) algorithm $R\colon Z \to \{0,1\}^\ell$ extracts $\ell$ bits. For a dataset $S \in Z^n$, an $\COMM_S$ oracle takes as an input an index $i$ and an algorithm $R$ and outputs a random value $w$ obtained by applying $R(z_i)$. A non-interactive algorithm is $\ell$-bit communication bounded if it accesses $S$ only via the $\ell$-bit $\COMM_S$ oracle each sample is accessed once and all of its queries are determined before observing any of the oracle's responses.
\end{definition}

As first shown by \citet{Ben-DavidD98}, it is easy to simulate a single query to $\COMM$ applied to a random sample from distribution $P$ using a single query to $\STAT_P(\tau)$. The simulation has been strengthened in \citep{FeldmanGRVX:12} and generalized to the $\COMM$ oracle that can access each sample more than once in \citep{SteinhardtVW16}.
\begin{theorem}[\citep{SteinhardtVW16}]
\label{thm:COMM-2-SQ}
Let $\A$ be a non-interactive $\ell$-bit communication bounded algorithm that makes queries to $\COMM_S$ for $S\in Z^n$ drawn i.i.d.~from some distribution $P$. Then for every $\delta >0$, there is an SQ algorithm $\A_{SQ}$ that makes $2n\ell$ non-adaptive queries to $\STAT_P\left(\delta/(2^{\ell+1} n)\right)$ and produces the same output as $\A$ with probability at least $1-\delta$. 
\end{theorem}

A direct corollary of Theorems \ref{thm:sq} and \ref{thm:COMM-2-SQ} is the following lower bound:
\begin{corollary}
	Fix $\gamma \in (0,1/2)$, $r \in (0,1)$ and $d \ge \gamma^{-2-2r/5}$. Let $\A$ be a non-interactive $\ell$-bit communication bounded algorithm with $n$ users. Assume that for any classification problem $(D,f^*)$ over $\mathbb{R}^d$ with margin $\gamma(f^*,D) \ge \gamma$, the algorithm outputs a hypothesis $\hat{f}$ with expected loss $\E_{\A}[\err_{f^*,D}(\hat{f})] \le 1/2 -\gamma^{1-r}$. Then, either $\ell \ge c \gamma^{-2r/5}$ or
	$
	n \ge \exp(c \gamma^{-2r/5})
	$, where $c > 0$ is a constant depending only on $r$.
\end{corollary}
The lower bound for learning linear models with convex loss can be extended analogously.
\fi
\appendix
\iffull
\printbibliography
\else
\bibliographystyle{ACM-Reference-Format}
\bibliography{vf-allrefs-local,adaptive}
\fi
\end{document}